\documentclass{article}

\usepackage[preprint]{nips_2018}

\usepackage[utf8]{inputenc} 
\usepackage[T1]{fontenc}    
\usepackage{hyperref}       
\usepackage{url}            
\usepackage{booktabs}       
\usepackage{amsfonts}       
\usepackage{nicefrac}       
\usepackage{microtype}      

\usepackage{amsmath, amssymb}
\usepackage{verbatim}
\usepackage{enumitem}
\usepackage{mathrsfs}
\usepackage{amsfonts,dsfont}
\usepackage{tikz}
\usepackage{todonotes}
\usepackage{amsthm}
\usepackage{mathtools}
\usepackage{caption}
\usepackage{graphicx}
\usepackage{subfig}
\usepackage{etoolbox}  
\newtoggle{long-version}
\toggletrue{long-version}

\newcommand*\circled[1]{\tikz[baseline=(char.base)]{
            \node[shape=circle,draw,inner sep=2pt] (char) {#1};}}
            
\DeclareMathAlphabet{\mathpzc}{OT1}{pzc}{m}{it}

\newcommand{\bE}{\mathbb{E}}
\newcommand{\RegretC}{\Delta''}
\newcommand{\RegretNc}{\Delta'}
\newcommand{\Regret}{\Delta}
\newcommand{\Defined}{\coloneqq}
\newcommand{\OptAvgReward}{\rho^*}
\newcommand{\MeanReward}{\bar{r}}

\newcommand{\MDP}{M}
\newcommand{\SMDP}{\mathbf{M}}
\newcommand{\setMDP}{\mathbb{S}}
\newcommand{\PlausibleMDP}{\mathpzc{M}}

\newcommand{\NumBatches}{\mathscr{B}}
\newcommand{\CountInBatch}{\mathscr{N}}
\newcommand\given[1][]{\:#1\vert\:}
\newcommand{\alg}{\mathfrak A}
\newcommand{\ucrl}{\textsc{Ucrl2}}
\newcommand{\ucrlR}{\textsc{Ucrl2-R}}
\newcommand{\ucrlRW}{\textsc{Ucrl2-RW}}
\newcommand{\swucrl}{\textsc{SW-Ucrl}}
\newcommand{\ucb}{\textsc{UCB}}
\newcommand{\set}[1]{\mathcal{#1}}
\newcommand{\sS}{\set{S}}
\newcommand{\sA}{\set{A}}

\newcommand{\EpInBatch}{E}

\newcommand{\hpx}[4]{\ensuremath{\hat p_{#1}\left ( {#2} | {#3}, {#4} \right ) }}

\newcommand{\tps}[3]{\tilde{{p}}_{#1}\left ( \cdot | {#2}, {#3} \right ) }
\newcommand{\hps}[3]{\hat{{p}}_{#1}\left ( \cdot | {#2}, {#3} \right )}
\newcommand{\hrs}[3]{\ensuremath{\hat r_{#1}\left ( {#2}, {#3} \right ) }}
\newcommand{\tpi}[1]{\tilde\pi_{#1}}
\newcommand{\M}{M}
\newcommand{\sM}{\mathcal{M}}
\newcommand{\tM}{\tilde M}
\newcommand{\sMk}{\sM_{k}}
\newcommand{\tMk}{\tM_{k}}

\newcommand{\trho}{\ensuremath{\tilde\rho}}

\newcommand{\tempVar}{\alpha}

\newcommand{\ind}{\mathds{1}}

\newcommand{\changepoint}{c}
\newcommand{\rewardfunction}{F}

\newcommand{\bigO}{\mathcal{O}}

\newcommand{\avgReward}{\rho}
\newcommand{\optAvgReward}{\avgReward^*}

\newtheorem{myTheorem}{Theorem}
\newtheorem{myLemma}{Lemma}
\newtheorem{myCorollary}{Corollary}
\newtheorem{myDefinition}{Definition}
\newtheorem{myProposition}{Proposition}

\title{A Sliding-Window Algorithm for Markov Decision Processes with Arbitrarily Changing Rewards and Transitions}

\author{
  Pratik Gajane \\
  Montanuniversität Leoben\\
  Erzherzog Johann-Strasse 3\\
  8700 Leoben, Austria \\
  \texttt{pratik.gajane@unileoben.ac.at} \\
   \And
   Ronald Ortner \\
  Montanuniversität Leoben\\
  Erzherzog Johann-Strasse 3\\
  8700 Leoben, Austria \\
  \texttt{ronald.ortner@unileoben.ac.at } \\
   \AND
  Peter Auer \\
  Montanuniversität Leoben\\
  Erzherzog Johann-Strasse 3\\
  8700 Leoben, Austria \\
  \texttt{auer@unileoben.ac.at }
}

\begin{document}

\maketitle

\begin{abstract}
 We consider reinforcement learning in changing Markov Decision Processes where \textit{both} the state-transition probabilities and the reward functions may vary over time. For this problem setting, we propose an algorithm using a sliding window approach and provide performance guarantees for the regret evaluated against the optimal non-stationary policy. We also characterize the optimal window size suitable for our algorithm. These results are complemented by a sample complexity bound on the number of sub-optimal steps taken by the algorithm. Finally, we present some experimental results to support our theoretical analysis.  
\end{abstract}

\section{Introduction}
\label{sec:Int}
A classical Markov Decision Process (MDP) provides a formal description of a sequential decision making problem. 
Markov decision processes are a standard model for problems in decision making with uncertainty (\citet{Puterman1994}, \citet{Bertsekas:1996:NP:560669}) and in particular for reinforcement learning. In the classical MDP model, the uncertainty is modeled by stochastic state-transition dynamics and reward functions, which however remain fixed throughout. Unlike this, here we consider a setting in which both the transition dynamics and the reward functions are allowed to change over time. As a motivation, consider the problem of deciding which ads to place on a webpage. The instantaneous reward is the payoff when viewers are redirected to an advertiser, and the state captures the details of the current ad. With a heterogeneous group of viewers, an invariant state-transition function cannot accurately capture the transition dynamics. The instantaneous reward, dependent on external factors, is also better represented by changing reward functions. For more details of how this particular example fits our model, cf.\ \citet{Yu2009a}, which studies a similar MDP problem, as well as \cite{Yu2009ab} and \citet{NIPS2013_4975} for additional motivation and further practical applications of this problem setting.        


\subsection{Main contribution}
For the mentioned \textit{switching-MDP} problem setting in which an adversary can make abrupt changes to the transition probabilities and reward distributions a certain number of times, we provide an algorithm called \swucrl, a version of \ucrl\ (\citet{Jaksch:2010:NRB:1756006.1859902}) that employs a sliding window to quickly adapt to potential changes. We derive a high-probability upper bound on the cumulative regret of our algorithm of $ \bigO \left( l^{1/3} T^{2/3} D^{2/3} S^{2/3} \left(A\log{ \left( \frac{T}{\delta} \right)} \right)^{1/3} \right)$  when the window size is adapted to the problem setting, including the number of changes. 
This improves upon the upper bound for \ucrl \ with restarts (\citet{Jaksch:2010:NRB:1756006.1859902}) for the same problem in terms of dependence on $D$, $S$ and~$A$. Moreover, our algorithm also works without the knowledge of the number of changes, although with a more convoluted regret bound, which shall be specified later. 

\subsection{Related work}
There exist several works on reinforcement learning in finite (non-changing) MDPs, including \citet{Burnetas:1997:OAP:265654.265664}, \citet{Bartlett:2009:RRB:1795114.1795119}, \citet{Jaksch:2010:NRB:1756006.1859902} to mention only a few. 
MDPs in which the state-transition probabilities change arbitrarily but the reward functions remain fixed, have been considered by \citet{Nilim:2005:RCM:1246500.1246504}, \citet{DBLP:conf/nips/XuM06}. 
On the other hand, \citet{NIPS2004_2730} and \citet{Dick2014} consider the problem of MDPs with fixed state-transition probabilities and changing reward functions. Moreover, \citet[Theorem 11]{NIPS2004_2730} also show that the case of MDPs with both changing state-transition probabilities and changing reward functions is computationally hard. \citet{Yu2009a} and \citet{Yu2009ab} consider arbitrary changes in the reward functions and arbitrary, but bounded, changes in the state-transition probabilities. They also give regret bounds that scale with the proportion of changes in the state-transition kernel and which in the worst case grow linearly with time. \citet{NIPS2013_4975} consider MDP problems with (oblivious) adversarial changes in state-transition probabilities and  reward functions and provide algorithms for minimizing the regret with respect to a comparison set of stationary (expert) policies. The MDP setting we consider is similar, however our regret formalization is different, in the sense that we consider the regret against an optimal non-stationary policy (across changes). This setting has already been considered by \citet{Jaksch:2010:NRB:1756006.1859902} and we use the suggested \ucrl\ with restarts algorithm as a benchmark to compare our work with.   

Sliding window approaches to deal with changing environments have been considered in other learning problems, too. In particular, \citet{Garivier:2011:UBP:2050345.2050365} consider the problem of changing reward functions for multi-armed bandits and provide a variant of \ucb (\citet{Auer:2002:FAM:599614.599677}) using a sliding window. 


\subsection{Outline}
The rest of the article is structured as follows. In Section \ref{sec:Setting}, we formally define the problem at hand. This is  followed by our algorithmic solution, \swucrl, presented in Section \ref{sec:Setting}, which also features regret bounds and a sample complexity bound. Next, in Section \ref{sec:EProof}, we analyze our algorithm providing proofs for the regret bound. Section \ref{sec:Exp} provides some complementing experimental results followed by some concluding discussion in Section \ref{sec:last} .
   
\section{Problem setting}
\label{sec:Setting}
In an MDP $\MDP(\sS, \sA, p, \rewardfunction)$ with finite state space $\sS$ (S = $|\sS|$) and a finite action space $\sA$ (A = $|\sA|$), the learner's task at each time step $t$ is to choose an action $a = a_t \in \sA$ to execute in the current state $s = s_t \in \sS$. Upon executing the chosen action $a$ in state $s$, the learner receives a reward $r_t$ given by some reward function $\rewardfunction$. Here, we assume that $\rewardfunction$ returns a value drawn iid from some unknown distribution on $[0,1]$ with mean $\MeanReward(s,a)$ and the environment transitions into the next state $s' \in \sS$ selected randomly according to the unknown probabilities $p(s' \given s, a)$.

In this article, we consider a setting in which reward distributions and state-transition probabilities are allowed to change (but not the state space and action space) at unknown time steps (called change-points henceforth). We call this setting a switching-MDP problem(following the naming of a similar MAB setting by \citet{Garivier:2011:UBP:2050345.2050365}). Neither the change-points nor the changes in reward distributions and state transition probabilities depend on the previous behavior of the algorithm or the filtration of the history $(s_1, a_1, r_1,..., s_t, a_t, r_t)$. It can be assumed that the change points are set in advance at time steps $\changepoint_1, \dots, \changepoint_l$ by an \textit{oblivious adversary}. 
At time step $t < \changepoint_1$, a switching-MDP $\SMDP$ is in its initial configuration $\MDP_0(\sS, \sA, p_0, \rewardfunction_0)$ where rewards are drawn from an unknown distribution on $[0,1]$ with mean $\MeanReward_0(s,a)$ and state transition occurs according to the transition probabilities $p_0(s' \given s, a)$. At time step $ \changepoint_i \leq t < \changepoint_{i+1} $, a switching-MDP $\SMDP$ is in configuration $\MDP_i(\sS, \sA, p_i, \rewardfunction_i)$. Thus, a switching-MDP problem $\SMDP$ is completely defined by a tuple $ (\setMDP = (\MDP_0,\dots,\MDP_l), \changepoint = (\changepoint_1,\dots,\changepoint_l))$.

An algorithm $\alg$ attempting to solve a switching-MDP $\SMDP =(\setMDP = (\MDP_0,\dots,\MDP_l), \changepoint = (\changepoint_1,\dots,\changepoint_l))$ from an initial state $s_1$ chooses an action $a_t$ to execute at time step $t$, i.e.\ it finds a policy $\pi : s_t \rightarrow a_t $. A policy $\pi$ can either choose the same action for a particular state at any time step (stationary policy), or it might choose different actions for the same state when it is visited at different time steps (non-stationary policy). The sequence of the states $s_t$ visited by $\alg$ at step $t$ as decided by its policy $\pi$, the action chosen $a_t$ and the subsequent reward $r_t$ received for $t=1, \dots$ can be be thought of as a result of stochastic process. 

As a performance measure, we use \textit{regret} which is used in various other learning paradigms as well. In order to arrive at the definition of the regret of an algorithm $\alg$ for a switching-MDP $\SMDP$, let us define a few other terms. The average reward $\avgReward$ for a constituent MDP $\MDP_i$ is the limit of the expected average accumulated reward when an algorithm $\alg$ following a stationary policy is run on $\MDP_i$ from an initial state $s$.
$$
\avgReward(\MDP_i, \alg, s) \Defined \lim_{T \rightarrow \infty} \frac{1}{T} \bE \left[ \text{Sum of rewards obtained from $1$ to $T$ on MDP $M_i$ by $\alg$} \right]
$$
We note that for a given (fixed) MDP the optimal average reward is attained by a stationary policy and cannot be increased by using non-stationary policies.

Another intrinsic parameter for MDP configuration $\MDP_i$ is its \textit{diameter}.  
\begin{myDefinition}(Diameter of a MDP)
The diameter of a MDP $\MDP_i$ is defined as follows:
$$
D(\MDP_i) = \max_{s_1,s_2 \in S, s_1 \neq s_2} \min_{\pi \in \Pi} \bE [\tau(s_1, s_2, \MDP_i,\pi)],
$$
where the random variable $\tau(s1, s_2, \MDP_i, \pi)$ denotes the number of steps needed to reach state $s_2$ from state $s_1$ in an MDP $\MDP_i$ for the first time following any policy from the set $\Pi$ of feasible stationary policies.   
\end{myDefinition}
For MDPs with finite diameter, the optimal average reward $\optAvgReward$ does not depend on the initial state (\citet{Puterman1994}). Thus, assuming finite diameter for all the constituent MDPs of a switching-MDP problem, $\optAvgReward_i $ for constituent MDP $\MDP_i$ is defined as 
$$
\optAvgReward_i \Defined \max_{\pi, s \in \sS} \avgReward(\MDP_i, \alg, s).
$$
With the above in hand, we can state that the regret of an algorithm $\alg$ for a switching-MDP problem is the sum of the missed rewards compared to the $l+1$ optimal average rewards $\optAvgReward_i$'s when the corresponding constituent MDP $\MDP_i$ is active. 

\begin{myDefinition}(Regret for a switching-MDP problem)
The regret of an algorithm $\alg$ operating on a switching-MDP problem $\SMDP$ = $\{\setMDP = \{\MDP_0,\dots,\MDP_l\}, \changepoint = \{\changepoint_1,\dots,\changepoint_l\}\}$ and starting at an initial state $s$ is defined 
$$
\Regret(\SMDP, \alg, s, T) = \sum_{t-1}^{T} \left( \optAvgReward_\SMDP(t) - r_t \right).
$$
where, $ \optAvgReward_\SMDP(t) \Defined \optAvgReward_i $ if $\MDP_i$ is active at time $t$.
\end{myDefinition}

When it is clear from the context, we drop the subtext $\SMDP$ and simply use $\optAvgReward(t)$ to denote $\optAvgReward_\SMDP(t)$. 

\section{Proposed algorithm: SW-UCRL}
\label{sec:Alg}
Our proposed algorithm, called Sliding Window UCRL (\swucrl) is a non-trivial modification of the \ucrl \ algorithm given by \citet{Jaksch:2010:NRB:1756006.1859902}. Unlike \ucrl, our algorithm \swucrl \ only maintains history of the last $W$ (called, window size) time steps. In a way, it could interpreted as \swucrl \ slides a window of size $W$ across the filtration of history. 

\begin{figure}[t]
\begin{center}
\fbox{\begin{minipage}{0.94\columnwidth}
\begin{description}
\smallskip
    \item[\quad Input:]  A confidence parameter $\delta \in (0,1)$, $\sS$, $\sA$ and window size $W$.
    \item[\quad Initialization:]  Set $t:=1$, and observe the initial state $s_1$.
    \item[\quad For] episodes $k=1,2,\ldots$ \textbf{do} \smallskip\\ 
        \textbf{Initialize episode }$k$: 
        \begin{enumerate}[nolistsep]
            \item Set the start time of episode $k$, $t_k:=t$.
            \item \label{algo:setN} For all $(s,a)$ in $\sS\times\sA$
                  initialize the state-action counts for episode $k$,
                  $v_k(s,a) := 0$.
				Further, set the the number of times any action action $a$ was executed in state~$s$ in $W$ time steps prior to episode $k$ for all the states $s \in \sS$ and actions $a \in \sA$,
               \begin{equation*}
                  N_{k}\left (s,a \right) := 
                         \#\left\{
                             t_k - W  \leq \tau < t_k : s_\tau = s, a_\tau = a
                           \right\}.
               \end{equation*}
            \item \label{algo:setP}\label{algo:compEst}
               For all $s,s'\in\sS$ and $a\in\sA$,
                  set the observed cumulative rewards when action $a$ was executed in state $s$ and
                  the number of times that resulted into the next state being $s'$ during $W$ time steps prior to episode $k$,
               \begin{equation*} 
                 R_{k}\left (s,a\right ) := 
                   \sum_{ \tau=t_k - W}^{t_k-1} r_\tau \ind\{s_\tau=s, a_\tau=a\},
               \end{equation*}
               \begin{equation*}
                 P_{k}\left (s,a,s^\prime \right ) := 
                         \#\left\{
                 t_k - W  \leq \tau < t_k : s_\tau = s, a_\tau = a, s_{\tau+1} = s^\prime
                         \right\}.
               \end{equation*}
               Compute estimates 
		 $\hrs{k}{s}{a} := \frac{R_{k}(s,a)}{\max\{1,N_k(s,a)\}},$
		   $\hpx{k}{s'}{s}{a} :=
                      \frac{P_{k}(s,a,s^\prime )}{\max\{1,N_k(s,a)\}}$.
        \end{enumerate}
        \textbf{Compute policy} $\tpi{k}$: 
        \begin{enumerate}[nolistsep,resume]
            \item \label{alg:plausible}
            Let $\sMk$ be the set of all MDPs with
              state space $\sS$ and action space $\sA$, and with transition probabilities $\tps{}{s}{a}$ close to~$\hps{k}{s}{a}$,
              and rewards $\tilde r(s,a)\in[0,1]$ close to~$\hrs{k}{s}{a}$,
              that is,
         \begin{eqnarray}\textstyle
          \label{eq:civR}
            \big\vert
              \tilde r(s, a) - \hrs{k}{s}{\vphantom{X^X_X} a}
            \big \vert \;\;
          & \leq & 
            \sqrt{ \tfrac{7\log\left( 2 S A t_k / \delta  \right )}
              {2\max\{1,N_k(s,a)\}}
            }
            \quad \textrm{ and }
            \\
          \label{eq:civ}
            \Big\Vert
              \tps{}{s}{\vphantom{X^X_X} a} - \hps{k}{s}{\vphantom{X^X_X} a}
            \Big \Vert_1
          & \leq &
            \sqrt{ \tfrac{14S\log\left( 2 A t_k / \delta  \right )}
              {\max\{1,N_k(s,a)\}}
            }
          \;.
          \end{eqnarray}
            \item \label{alg:chooseMkAndPik}
                Use extended value
                  iteration to 
		find a policy $\tpi{k}$ and an optimistic MDP~$\tMk \in \sMk$ such that 
		\begin{equation*} 
			\trho_k := \min_s \rho(\tMk,\tpi{k},s)\geq \max_{\M'\in\sMk,\pi,s'} \rho(M',\pi,s')  - \frac{1}{\sqrt{t_k}}.
		\end{equation*}
        \end{enumerate}
        \textbf{Execute policy }$\tpi{k}$: 
        \begin{enumerate}[nolistsep,resume]
            \item\label{alg:doEpisode} \textbf{While} $v_k(s_{t},\tpi{k}(s_t)) < \max\{1,N_{k}(s_{t},\tpi{k}(s_{t}))\}$ \textbf{do} 
            \begin{enumerate}[nolistsep]
                \item Choose action $a_{t} = \tpi{k}(s_{t})$, 
                 obtain reward $r_t$, and\\ observe next state~$s_{t+1}$.	  
                \item Update $v_k(s_t,a_t) := v_k(s_t,a_t) + 1$.  
                \item Set $t:=t+1$.\smallskip
             \end{enumerate}
        \end{enumerate}
\end{description}
\end{minipage}}
\caption{\label{f:swucrl}The \swucrl\ algorithm.}
\end{center}
\vspace{-0.64cm}
\end{figure}

At its core, \swucrl \ works on the principle of ``optimism in the face of uncertainty''. It proceeds in episodes divided into three phases as its predecessor \ucrl. At the start of every episode~$k$, it assesses its performance in the past $W$ time-steps and changes the policy, if necessary. More precisely (see Figure \ref{f:swucrl}), during the initialization phase for episode $k$ (steps $1$, $2$ and $3$), it computes the estimates $\hrs{k}{s}{a}$ and $\hpx{k}{s'}{s}{a}$ for mean rewards for each state-action pair $(s,a)$ and the state-transition probabilities for each triplet $(s,a,s')$ from the last $W$ observations. In the policy computation phase (steps $4$ and $5$), \swucrl\ defines a set of MDPs $\sMk$ which are statistically plausible given $\hrs{k}{s}{a}$ and $\hpx{k}{s'}{s}{a}$. The mean rewards and the state-transition probabilities of every MDP in $\sMk$ are stipulated to be \textit{close} to the estimated mean rewards $\hrs{k}{s}{a}$ and estimated state-transition probabilities $\hpx{k}{s'}{s}{a}$, respectively. The corresponding confidence intervals are specified in Eq.\ (\ref{eq:civR}) and Eq. (\ref{eq:civ}). The algorithm then chooses an optimistic MDP $\tMk$ from $\sMk$ and uses extended value iteration \citep{Jaksch:2010:NRB:1756006.1859902} to select a near-optimal policy $\tpi{k}$ for $\tMk$. In the last phase of the episode (step $6$), $\tpi{k}$ is executed. The lengths of the episodes are not fixed a priori, but depend upon the observations made so far in the current episode as well as the $W$ observations before the start of the episode. Episode $k$ ends when the number of occurrences $v_k(s,a)$ of the current state-action pair $(s,a)$ in the episode is equal to the number of occurrences $N_{k}(s,a)$ of the same state-action pair $(s,a)$ in $W$ observations before the start of episode $k$. It is worth restating that the values $N_{k}(s,a)$, $\hrs{k}{s}{a}$, and $\hpx{k}{s'}{s}{a}$ are computed only from the previous $W$ observations at the start of each episode. Not considering observations beyond $W$ is done with the intention of ``forgetting" previously active MDP configurations. 
Note that due to the episode termination criterion no episode can be longer than $W$ steps.
%

The following theorem provides an upper bound on the regret of \swucrl. The elements of its proof can be found in Section \ref{sec:EProof}.
\begin{myTheorem}\label{Thm:UC_SWUCRL}
Given a switching-MDP with $l$ changes in the reward distributions and state-transition probabilities,
with probability at least $1 - \delta$, it holds that for any initial state $s \in \sS$ and any $ T \geq \max{(8 \delta, 2A\delta)}$, the regret of \swucrl \ using window size $W \geq \max{ \left( SA, \frac{A \left( \log_2{(8W/SA)} \right)^2 }{\log{(T/\delta)}}\right) } $ is bounded by 
$$
 2lW + 66.12 \left\lceil \frac{T}{\sqrt{W}} \right\rceil DS \sqrt{A\log{ \left( \frac{T}{\delta} \right)}},
$$
where $D = \max\{D(\MDP_0), \dots, D(\MDP_{l+1})\}$.
\end{myTheorem}

From above, one can compute the optimal value of $W$ as follows:
\begin{equation}
\label{Eq:OptWindow}
W^* = \left(\frac{16.53}{l} T DS \sqrt{A\log{ \left( \frac{T}{\delta} \right)}}\right)^{2/3}
\end{equation}

If the time horizon $T$ and the number of changes $l$ are known to the algorithm, then $W$ can be set to its optimal value given by Eq. (\ref{Eq:OptWindow}), and we get the following bound.

\begin{myCorollary}\label{Cor:OptWindow}
Given a switching-MDP problem with $D = \max\{D(\MDP_0), \dots, D(\MDP_{l+1})\}$ and $l$ changes in the reward distributions and state-transition probabilities, the regret of \swucrl\ using $W^* = \left(\frac{16.53}{l} T DS \sqrt{A\log{ \left( \frac{T}{\delta} \right)}}\right)^{2/3}$ for any initial state $s \in S$ and any $ T \geq \max{(8 \delta, 2A\delta)} $ is upper bounded by
$$
38.94 \cdot l^{1/3} T^{2/3} D^{2/3} S^{2/3} \left(A\log{ \left( \frac{T}{\delta} \right)} \right)^{1/3}
$$
with probability at least $1 - \delta$.
\end{myCorollary}

\iftoggle{long-version}{The proof of this corollary is detailed in Appendix~\ref{Proof:Cor:OptWindow}.}{For the proof of this corollary, we refer to the extended version of this article.}

This bound improves upon the bound provided for \ucrl \ with restarts (\citet[Theorem 6]{Jaksch:2010:NRB:1756006.1859902}) in terms of dependence of $D$, $S$ and $A$. Our bound features $D^{2/3}$, $S^{2/3}$ and $A^{1/3}$ while the provided bound for  \ucrl \ with restarts features $D$, $S$ and $A^{1/2}$. We note however that it might be be possible to get an improved bound for \ucrl \ with restarts using an optimized restarting schedule.  

Finally, we also obtain the following PAC-bound for our algorithm.
\begin{myCorollary}\label{Cor:SamComp}
Given a switching-MDP problem with $l$ changes,
with probability at least $1 - \delta$, the average per-step regret of \swucrl \ using $W^* = \left(\frac{16.53}{l} T DS \sqrt{A\log{ \left( \frac{T}{\delta} \right)}}\right)^{2/3}$ is at most $\epsilon$ after any $T$ steps with
$$
T \geq 2 \cdot (38.94)^3 \cdot\frac{ l D^2 S^2 A}{\epsilon^3} \log{\left(\frac{(38.94)^3 l D^2 S^2 A}{\epsilon^3\delta}\right)}.
$$
\end{myCorollary}


\iftoggle{long-version}{The proof of this corollary is detailed in Appendix~\ref{Proof:Cor:SamComp}.}{Please consult the extended version of this article for the proof of this corollary.}


\section{Analysis of Sliding Window UCRL}\label{sec:EProof}
The regret can be split up into two components: the regret incurred due to the changes 
in the MDP ($\RegretC$) and the regret incurred when the MDP remains the same ($\RegretNc$). Due to the definition of \swucrl, a change in the MDP can only affect the episode in which the said change has occurred or the following episode. Due to the episode stopping criterion, the length of an episode can at-most be equal to
the window size. Hence $\RegretC \leq 2lW$.

Now, we compute the regret in the episodes in which the MDP doesn't change. This computation is similar to the analysis of \textsc{Ucrl2} in \cite[Section 4]
{Jaksch:2010:NRB:1756006.1859902}. We define the regret in episode $k$ in which the switching-MDP doesn't change its configuration and only stays in configuration $M_i$ as
$$
\RegretNc_k \Defined \sum_{s,a} v_k(s,a)\left( \OptAvgReward(t_k) - \MeanReward(s,a)\right).
$$
Then ---now considering only episodes which are not affected by changes---, one can show that
\begin{equation}
\label{Eq:EpSplit}
\RegretNc \leq \sum_{k=1}^m \RegretNc_k + \sqrt{\frac{5}{2} T \log{\left(\frac{8T}{\delta}\right)}} 
\end{equation}
with probability at least $1 - \frac{\delta}{12T^{5/4}}$
where $m$ is the respective number of episodes up to time-step $T$.

Denoting the unchanged MDP in episode $k$ as $\MDP_k$
, with probability at least $1 - \frac{\delta}{12T^{5/4}}$,
\begin{equation}
\label{Eq:FailConfidence}
\sum_{k=1}^{m} \RegretNc_k \ind_{\MDP_k \notin \PlausibleMDP_k} \leq \sqrt{T}.
\end{equation}
Furthermore, as for derivation of \eqref{Eq:EpSplit} and \eqref{Eq:FailConfidence} following the proof of  \cite{Jaksch:2010:NRB:1756006.1859902}, one can show that 
\begin{align*}
\sum_{k=1}^{m} \RegretNc_k \ind_{\MDP_k \in \PlausibleMDP} \quad
&\leq \quad  D \sqrt{14S \log{ \left( \frac{2AT}{\delta}\right) }} \sum_{k=1}^{m} \sum_{s,a} \frac{v_k(s,a)}{\sqrt{\max{ \{ 1, N_k(s,a) \} }}} \\
& \qquad \quad + D \sqrt{\frac{5}{2} T \log{ \left( \frac{8T}{\delta} \right)}} + mD \\
& \qquad \quad + \left( \sqrt{14 \log{\left( \frac{2SAT}{\delta} \right)} } + 2 \right)  \sum_{k=1}^{m} \sum_{s,a} \frac{v_k(s,a)}{\sqrt{\max{ \{ 1, N_k(s,a) \} }}}.
\end{align*}

To proceed from here, we make use of the following novel lemmas which present some challenges related to handling the limitation of history to the sliding window. 
\begin{myLemma}
\label{Lemma:NumEpisodes}
Provided that $W\geq SA$, the number $m$ of episodes of \swucrl \ up to time-step $T \geq SA$ is upper bounded as 
$$
m \leq \left\lceil \frac{T}{W} \right\rceil  SA \log_2{ \left( \frac{8W}{SA} \right)}.
$$
\end{myLemma}

\iftoggle{long-version}{The proof for Lemma \ref{Lemma:NumEpisodes} is given in Appendix \ref{Proof:NumEpisodes}.}{Please consult the extended version of this article for the proof of this Lemma.} Here we only provide a key idea behind the proof. We argue that the number of episodes in a batch are maximum, if the state-action counts at the first step of the batch are all $0$. Summing up such maximal number of episodes for batches of size $W$ gives the claimed bound. 

\begin{myLemma}
\label{Lemma:vN}
$$\sum_{k=1}^{m} \sum_{s,a} \frac{v_k(s,a)}{\sqrt{\max{ \{ 1, N_k(s,a) \} }}} \leq (2\sqrt{2}  + 2) \left \lceil \frac{T}{W} \right \rceil \sqrt{SAW}.$$
\end{myLemma}

\iftoggle{long-version}{The detailed proof for Lemma \ref{Lemma:vN} is given in Appendix \ref{Proof:vN}. }{We refer to the extended version of this article for the detailed proof of this lemma.} Here, we provide a brief overview of the proof. 

\textit{Proof sketch.}
Divide the time horizon into batches such that first batch starts at $t=1$ and each batch ends with the earliest episode termination after the batch size reaches $W$. Then $W \leq $ size of each batch $\leq 2W$ and the number of batches $|\NumBatches| \leq \left\lceil \frac{T}{W} \right\rceil $. Let $N^{+}_k(s,a) \Defined \#(s,a)$ in the current batch when episode $k$ starts, $N^{-}_k(s,a) \Defined  N_k(s,a) - N^{+}_k(s,a)$, and $N^b(s,a) \Defined \#(s,a)$ in batch $b$. Then, $\sum_{s,a}N^b(s,a) \leq 2W$ and we have

\begin{align*}
 \sum_{k=1}^{m} \sum_{s,a} \frac{v_k(s,a)}{\sqrt{\max{ \{ 1, N_k(s,a) \} }}} &=  \sum_{s,a} \sum_{k=1}^{m}  \frac{v_k(s,a)}{\sqrt{\max{ \{ 1, N^{+}_k(s,a) + N^{-}_k(s,a) \} }}} \\
&\leq \sum_{s,a} \sum_{b=1}^{\NumBatches} \left( \sqrt{N^{b-1}(s,a)} + \left( \sqrt{2} + 1 \right) \sqrt{\sum_{k \in \EpInBatch_b} v_k(s, a) }\right) \\
&\leq \sum_{b=1}^{\NumBatches} \left( \sqrt{2SAW} +  \left( \sqrt{2} + 1 \right) \sqrt{2SAW} \right) \\
&\leq  \left\lceil \frac{T}{W} \right\rceil  \left( 2\sqrt{2} + 2 \right)\sqrt{SAW}.
\end{align*}
The first inequality follows from a proposition 
\iftoggle{long-version}{\ref{Prop1} given in Appendix \ref{sec:Prop1}}{ provided in the extended version of this article}, while the second inequality follows from Jensen's inequality. 
\qed

Using Lemma \ref{Lemma:NumEpisodes} and Lemma \ref{Lemma:vN}, we get that, with probability at least $1 - \frac{\delta}{12T^{5/4}}$


\begin{align}
\sum_{k=1}^{m} \Delta_k \ind_{\MDP_k \in \PlausibleMDP} \quad
&\leq \quad  D \sqrt{\frac{5}{2} T \log{ \left( \frac{8T}{\delta} \right)}} +  \left\lceil \frac{T}{W} \right\rceil  DSA \log_2{ \left( \frac{8W}{SA} \right)} \nonumber \\
& \qquad \quad + \left( 2D \sqrt{14S \log{\left( \frac{2AT}{\delta} \right) }} + 2\right) (2\sqrt{2} + 2) \left \lceil \frac{T}{W} \right \rceil \sqrt{SAW} .\label{Eq:SucConfidence}
\end{align}

Then, using Eq. (\ref{Eq:EpSplit}), (\ref{Eq:FailConfidence}) and (\ref{Eq:SucConfidence}), with probability at least $1 - \frac{\delta}{12T^{5/4}} 
- \frac{\delta}{12T^{5/4}} - \frac{\delta}{12T^{5/4}}$,

\begin{align}
\RegretNc &\leq \underbrace{\sqrt{\frac{5}{8} T \log{\left( \frac{8T}{\delta} \right)}} + \sqrt{T} + D\sqrt{\frac{5}{2} T \log{\left( \frac{8T}{\delta} \right)}}}_{\circled{$E_1$}} \nonumber \\
& \quad + \underbrace{\left\lceil \frac{T}{W} \right\rceil  DSA \log_2{ \left( \frac{8W}{SA} \right)} + \left( 2D \sqrt{14S \log{\left( \frac{2AT}{\delta} \right) }} + 2\right) (2\sqrt{2} + 2) \left \lceil \frac{T}{W} \right \rceil \sqrt{SAW}}_{\circled{$E_2$}} \label{Eq:DivParts},
\end{align}

\begin{equation}\label{Eq:ValPart1}
\mbox{ and } \circled{$E_1$} \leq 4.36 D \sqrt{T \log{\left(\frac{T}{\delta}\right)}}  \qquad \text{ at } T \geq 8 \delta ,
\end{equation} 

\begin{equation}
\circled{$E_2$} \leq 61.76 \left\lceil \frac{T}{\sqrt{W}} \right\rceil DS \sqrt{A \log{ \left( \frac{T}{\delta} \right)}} \quad \text{if }  T \geq 2A\delta \text{ and } W  \geq \frac{A \left( \log_2{(8W/SA)} \right)^2 }{\log{(T/\delta)}}. \label{Eq:ValPart2}
\end{equation}

For the claimed simplifications of $E_1$ and $E_2$, see 
\iftoggle{long-version}{Appendix \ref{sec:E1} and \ref{sec:E2} respectively.}{the extended version of this article.}
%
\noindent From Eq. (\ref{Eq:DivParts}), (\ref{Eq:ValPart1}), and (\ref{Eq:ValPart2}) and since $\sum_{T=2}^{\infty} \frac{\delta}{4T^{5/4}} < \delta  $, with probability at least $1 - \delta$, the regret incurred during the episodes in which the MDP doesn't change is
$$
 \RegretNc \leq 66.12 \left\lceil \frac{T}{\sqrt{W}} \right\rceil DS \sqrt{A\log{ \left( \frac{T}{\delta} \right)}}  \quad \text{if } T \geq \max{(8 \delta, 2A\delta)} \text{ and } W \geq \max{ \left( SA, \frac{A \left( \log_2{(8W/SA)} \right)^2 }{\log{(T/\delta)}}\right) }. 
$$
\noindent Adding the regret incurred in the episodes not affected by changes,
$$
\Regret(\SMDP, \alg, s, T) \leq 2lW + 66.12 \left\lceil \frac{T}{\sqrt{W}} \right\rceil DS \sqrt{A\log{ \left( \frac{T}{\delta} \right)}}.
$$


\section{Experiments}
\label{sec:Exp}
  \begin{figure}[t]
     \subfloat[Average regret plots for $2$ changes \label{fig:PlotwrtTpartA}]{%
       \includegraphics[width=0.5\textwidth]{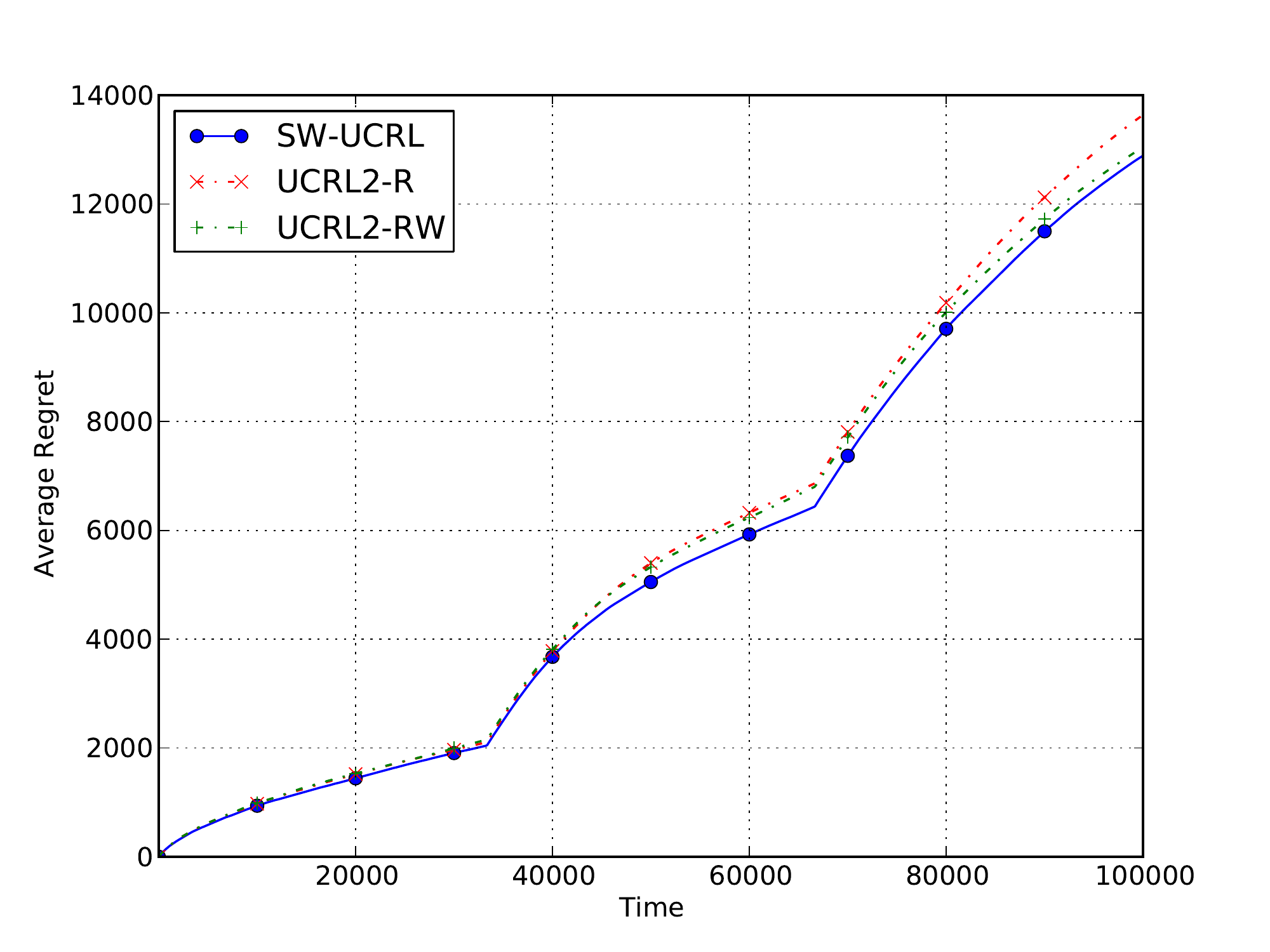}
     }
     \hfill
    \subfloat[Average regret plots for $4$ changes \label{fig:PlotwrtTpartB}]{%
       \includegraphics[width=0.5\textwidth]{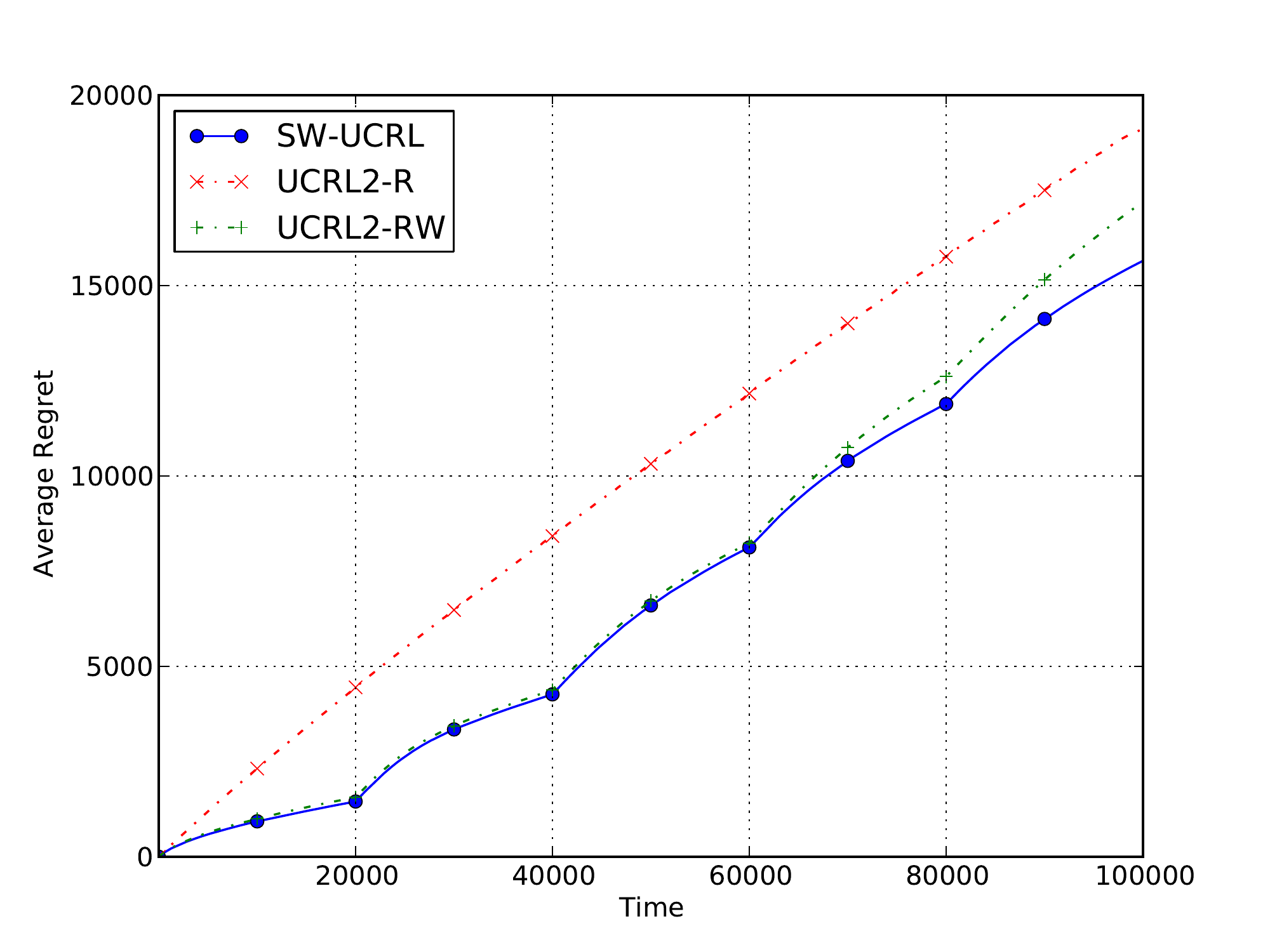}
     }     
     \caption{Average regret plots for switching-MDPs with $S = 5, A = 3$ and $T = 100000$}
     \label{fig:PlotwrtT}
	\end{figure}
    
For practical evaluation, we generated switching-MDPs with $S = 5, A = 3$, and $T = 100000$. The $l$ changes are set to happen at every $\lceil \frac{T}{l}\rceil$ time steps. This simple setting can be motivated from the ad example given in Section \ref{sec:Int} in which changes happen at regular intervals. 

For \swucrl, the window size was chosen to be the optimal as given by Eq.(\ref{Eq:OptWindow}), using a lower bound of $\log_{A}{S} - 3$ for the diameter. 
For comparison, we used two algorithms : \ucrl \ with restarts as given in \citet{Jaksch:2010:NRB:1756006.1859902} (referred to as \ucrlR \ henceforth) and \ucrl \ with restarts after every $W^*$ time steps (referred to as \ucrlRW \ henceforth). 
Note that the latter restarting schedule is a modification by us, not provided by \citet{Jaksch:2010:NRB:1756006.1859902}.
\swucrl, \ucrlR, and \ucrlRW \ were run with $\delta = 0.1$ on $1000$ switching-MDP problems with random rewards and state-transition probabilities.

Figure \ref{fig:PlotwrtTpartA} shows the average regret for $2$ changes and Figure \ref{fig:PlotwrtTpartB} for $4$ changes. A clearly noticeable trend in both plots (at least for \swucrl \ and our modification, \ucrlRW) are the ``bumps'' in regret curves at time steps where the changes occur.  That behaviour is expected as it shows that the algorithms were learning the MDP configuration indicated by the regret curves beginning to flatten, when a change to another MDP  results in an ascent of regret curves. \ucrlR, and \ucrlRW \ give only slightly worse performance when the number of changes are limited to $2$. However, even for a moderate number of changes as $4$, \swucrl \ and our modification, \ucrlRW \ are observed to give better performance than \ucrlR. In both cases, our proposed algorithm gives improved performance over \ucrlRW.



\section{Discussion and Further Directions}
\label{sec:last}
Theoretical performance guarantee and experimental results demonstrate that the algorithm introduced in this article, \swucrl, provides a competent solution for the task of regret-minimization on MDPs with arbitrarily changing rewards and state-transition probabilities. We have also provided a sample complexity bound on the number of sub-optimal steps taken by \swucrl. 

We conjecture that the sample complexity bound can be used to provide a variation-dependent regret bound, although the proof might present a few technical difficulties when handling the sliding window aspect of the algorithm. A related question is to establish a link between the extent of allowable variation in rewards and state-transition probabilities and the minimal achievable regret, as was done recently for the problem of multi-armed bandits with non-stationary rewards in \citet{NIPS2014_5378}. Another direction is to refine the episode-stopping criterion so that a new policy is computed only when the currently employed policy performs below a suitable reference value.

\newpage
\bibliographystyle{plainnat}
\bibliography{references} 


\iftoggle{long-version}{
\onecolumn
\newpage

\appendix
\renewcommand{\thesection}{\Roman{section}}
\renewcommand\thesubsection{\Alph{subsection}}
\section{Proof of Lemma \ref{Lemma:NumEpisodes}}
\label{Proof:NumEpisodes}
\begin{proof}
Divide the $T$ time steps into batches of equal size $W$ (with the possible exception of the last batch). For each of these batches we consider the maximal number of episodes contained in this batch. Obviously, the maximal number of episodes can be obtained greedily, if each episode is shortest possible. For each time step $t$ with state action counts $N(s,a)$ in the window reaching back to $t-W$, the shortest possible episode starting at $t$ (according to the episode termination criterion) will consist of $\max\{1,\min_{s,a} N(s,a)\}$ repeated visits to a fixed state-action pair contained in $\arg\min N(s,a)$.

Accordingly, in a window of size $W$, the number of episodes is largest, if the state-action counts at the first step of the batch are all $0$. For this case we know (cf.\ Lemma of \cite{Jaksch:2010:NRB:1756006.1859902}) that the number of episodes within $W$ steps is bounded by
$ 
SA \log_2{\left( \frac{8W}{SA}\right)}.
$
Summing up over all $\left \lceil \frac{T}{W} \right \rceil$ batches gives the claimed bound.

\end{proof}


\section{Technical Details for the proof of Lemma \ref{Lemma:vN}}
\label{Proof:vN}
\subsection{Proof of Lemma \ref{Lemma:vN}}
\begin{proof}
We shall prove this lemma by dividing the time horizon into of batches (different from those used in the proof of Lemma \ref{Lemma:NumEpisodes}) as follows.
The first batch starts at $t=1$ and each batch ends with the earliest episode termination after the batch size reached $W$. That way, each episode is completely contained in one batch.
As any episode can be at most of size $W$, it holds that $W \leq $ size of each batch $\leq 2W$. Therefore, the number of batches $|\NumBatches| \leq \left\lceil \frac{T}{W} \right\rceil $. 

Let $\EpInBatch_b$ be the set containing the episodes in batch b, and 
let $N^{+}_k(s,a) \Defined $ number of occurrences of state-action pair $(s,a)$ in the current batch when episode $k$ starts. Clearly $N^{+}_k(s,a) \leq N_k(s,a)$. Let $N^{-}_k(s,a) \Defined  N_k(s,a) - N^{+}_k(s,a)$. Furthermore, let $N^b(s,a) \Defined $ number of occurrences of state-action pair $(s,a)$ in batch $b$, setting $N^b(s,a):=0$. Note that $\sum_{s,a} N^b(s,a) \leq 2W$.

We have
\begin{align*}
\label{Eq:DivBatches}
 \sum_{k=1}^{m} \sum_{s,a} \frac{v_k(s,a)}{\sqrt{\max{ \{ 1, N_k(s,a) \} }}} &=  \sum_{s,a} \sum_{k=1}^{m}  \frac{v_k(s,a)}{\sqrt{\max{ \{ 1, N^{+}_k(s,a) + N^{-}_k(s,a) \} }}} \\
&= \sum_{s,a} \sum_{b=1}^{\NumBatches} \sum_{k \in \EpInBatch_b}  \frac{v_k(s,a)}{\sqrt{\max{ \{ 1, N^{+}_k(s,a) + N^{-}_k(s,a) \} }}} \\
&\leq \sum_{s,a} \sum_{b=1}^{\NumBatches} \left( \sqrt{N^{b-1}(s,a)} + \left( \sqrt{2} + 1 \right) \sqrt{\sum_{k \in \EpInBatch_b} v_k(s, a) }\right) \\
&=  \sum_{s,a} \sum_{b=1}^{\NumBatches} \left( \sqrt{N^{b-1}(s,a)} + \left( \sqrt{2} + 1 \right) \sqrt{ N^b(s,a) }\right) \\
&=  \sum_{b=1}^{\NumBatches} \left( \sum_{s,a} \sqrt{N^{b-1}(s,a)} + \sum_{s,a} \left( \sqrt{2} + 1 \right) \sqrt{ N^b(s,a) }  \right) \\ 
&\leq \sum_{b=1}^{\NumBatches} \left( \sqrt{2SAW} +  \left( \sqrt{2} + 1 \right) \sqrt{2SAW} \right) \\
&= \sum_{b=1}^{\NumBatches} \left( 2\sqrt{2} + 2 \right)\sqrt{SAW} \\
&\leq  \left\lceil \frac{T}{W} \right\rceil  \left( 2\sqrt{2} + 2 \right)\sqrt{SAW}
\end{align*}
In the above, the first inequality follows from using Proposition \ref{Prop1} with $n = |\EpInBatch_b|$, $z_k = v_k(s,a)$, $x_k = N^{+}_k(s,a)$, $y_k = N^{-}_k(s,a) $, and $Y = N^{b-1}(s,a)$, while the second inequality follows from Jensen's inequality. 
\end{proof}


\subsection{Proposition required to prove  Lemma \ref{Lemma:vN}}
\label{sec:Prop1}
\begin{myProposition}
\label{Prop1}
For any non-negative integers $x_1,\dots,x_n$, $z_1,\dots,z_n$ and $y_1,\dots ,y_n$ with the following properties
\begin{equation}
\label{Eq:propA}
x_k + z_k = x_{k+1}
\end{equation}
\begin{equation}
\label{Eq:propB}
x_k + y_k \geq z_k
\end{equation}
\begin{equation}
\label{Eq:propC}
y_n \leq y_{n-1} \dots y_2 \leq y_1 \leq Y
\end{equation}
\begin{equation}
\label{Eq:propD}
x_1 = 0
\end{equation}
\begin{equation}
\label{Eq:propE}
Z_k = \max\{1, x_k + y_k\}
\end{equation}
it holds that,
$$
\sum_{k=1}^{n} \frac{z_k}{\sqrt{Z_k}} \leq \sqrt{Y} + \left(\sqrt{2} + 1 \right) \sqrt{ \sum_{k=1}^{n}z_k } 
$$
\end{myProposition}

\begin{proof}
First note, using \eqref{Eq:propA} and \eqref{Eq:propD}, 
\begin{equation}
\label{Eq:summingup}
x_{k} = x_{k-1} + z_{k-1} = x_{k-2} + z_{k-2} + z_{k-1} = \dots = x_1 + \sum_{i=1}^{k-1}z_i = \sum_{i=1}^{k-1}z_i
\end{equation}

We now prove the proposition by induction over $n$. \\
\textit{Base case}:\\
$$\frac{z_1}{\sqrt{Z_1}} \leq \frac{x_1 + y_1}{\sqrt{\max\{1, x_1 + y_1\}}} = \frac{y_1}{ \sqrt{ \max\{ 1,y_1\}}} \leq \sqrt{Y} + (\sqrt{2} + 1)\sqrt{z_1} $$

The first equality is true because $x_1 = 0$ and the last inequality is true because
\begin{itemize}
\item if $y_1=0$, then $\max\{1, y_1\} = 1$, and $\frac{y_1}{ \sqrt{ \max\{ 1,y_1\}}} = 0 $ and the RHS is non-negative since all $z_1$ and $y_1,\dots, y_n$ are non-negative integers. 
\item if $y_1 \geq 1$, then $\max\{1, y_1\} = y_1$ and $\frac{y_1}{ \sqrt{ \max\{ 1,y_1\}}} = \sqrt{y_1} \leq \sqrt{Y}$ using \eqref{Eq:propC}.
\end{itemize}

\textit{Inductive step:} 
\begin{align*}
\sum_{k=1}^{n} \frac{z_k}{\sqrt{Z_k}} &= \sqrt{Y} + \left(\sqrt{2} + 1 \right)  \sqrt{ \sum_{k=1}^{n-1}z_k}  + \frac{z_n}{\sqrt{Z_n}} \\
&= \sqrt{Y} + \sqrt{ \left(\sqrt{2} + 1 \right)^2 \sum_{k=1}^{n-1} z_k + 2 \left(\sqrt{2} + 1 \right) \frac{ \left(\sum_{k=1}^{n-1} z_k \right) z_n}{ \sqrt{Z_n}} + \frac{z_n^2}{Z_n} } \\
&= \sqrt{Y} + \sqrt{ \left(\sqrt{2} + 1 \right)^2  \sum_{k=1}^{n-1} z_k  + 2 \left(\sqrt{2} + 1 \right) \frac{ \left(\sum_{k=1}^{n-1} z_k \right) z_n}{ \sqrt{\max\{1, x_n + y_n\}}} + \frac{z_n^2}{\max\{1, x_n + y_n\}} } \\
&\leq \sqrt{Y} + \sqrt{ \left(\sqrt{2} + 1 \right)^2  \sum_{k=1}^{n-1} z_k  + 2 \left(\sqrt{2} + 1 \right)z_n + z_n} \\
&= \sqrt{Y} + \sqrt{ \left(\sqrt{2} + 1 \right)^2  \sum_{k=1}^{n-1} z_k  +  (\sqrt{2} + 1)^2 z_n} \\
&= \sqrt{Y} +  \left(\sqrt{2} + 1 \right)  \sqrt{ \sum_{k=1}^{n} z_k } 
\end{align*}
In the above, the first inequality is true because,
\begin{itemize}
\item if ${\max\{1, x_n + y_n\}} = 1$,
then $ z_n \leq x_n + y_n \leq 1$. Therefore, $\frac{z_n^2}{\max\{1, x_n + y_n\}} = z_n^2 \leq z_n$ and
$\frac{\left(\sum_{k=1}^{n-1} z_k \right)}{\sqrt{\max\{1, x_n + y_n\}}} = \sum_{k=1}^{n-1} z_k = x_n \leq 1 $ using \eqref{Eq:summingup} 
\item if ${\max\{1, x_n + y_n\}} = x_n + y_n$ then $\frac{z_n^2}{\max\{1, x_n + y_n\}} = \frac{z_n^2}{x_n + y_n} \leq z_n$ using \eqref{Eq:propB}
\end{itemize}
\end{proof}

\section{Proof of Corollary \ref{Cor:OptWindow}}
\label{Proof:Cor:OptWindow}
\begin{proof}
\begin{align*}
\Regret(\SMDP, \alg, s, T) &\leq 2lW^* + 66.12 \left\lceil \frac{T}{\sqrt{W^*}} \right\rceil \psi \\
						 &\leq 2 \cdot l^{1/3} \cdot \left( 16.53 \cdot T  \psi \right)^{2/3} +  66.12 \cdot \frac{T\psi l^{1/3}}{(16.53 \cdot T  \psi)^{1/3}} \qquad \text{(Using Eq. \ref{Eq:OptWindow})} \\
                         &= 2(16.53)^{2/3} \cdot l^{1/3} T^{2/3}  \psi^{2/3} + \frac{66.12}{16.53^{1/3}}\cdot l^{1/3} T^{2/3}  \psi^{2/3} \\
                         &\leq 38.94 \cdot l^{1/3} T^{2/3} D^{2/3} S^{2/3} \left(A\log{ \left( \frac{T}{\delta} \right)} \right)^{1/3}	
\end{align*}
\end{proof}

\section{Proof of Corollary \ref{Cor:SamComp}}
\label{Proof:Cor:SamComp}
\begin{proof}
The proof uses a key idea from \citet[Corollary 3]{Jaksch:2010:NRB:1756006.1859902}. Let $T_0$ be such that for any $T \geq T_0$, the average per-step regret of \swucrl \ using $W^*$ is at-most $\epsilon$. Therefore, according to Corollary \ref{Cor:OptWindow},
\begin{align}
\epsilon T &\geq 38.94 \cdot l^{1/3} T^{2/3} D^{2/3} S^{2/3} \left(A\log{ \left( \frac{T}{\delta} \right)} \right)^{1/3}  \nonumber \\
T &\geq \tempVar \log{(T/\delta)} \qquad \qquad \text{where } \tempVar \Defined \frac{(38.94)^3 l D^2 S^2 A}{\epsilon^3} \label{eq:Cor2eq1}
\end{align}
Assume that $T_0 = 2 \tempVar\log(\tempVar / \delta)$ for . Then,
\begin{align}
T_0 &= \tempVar\log((\tempVar / \delta)^2) \nonumber \\
    &> \tempVar \log(2 \tempVar \log{(\tempVar / \delta)} / \delta) \qquad \qquad \text{Using } x > 2\log{x} \label{} \nonumber \\
    &= \tempVar \log{(T_0/\delta)} \label{eq:Cor2eq2}
\end{align}
From Eq. \ref{eq:Cor2eq1} and Eq. \ref{eq:Cor2eq2}, it is clear that,
$$ T_0 \geq 2 \cdot (38.94)^3 \cdot\frac{ l D^2 S^2 A}{\epsilon^3} \log{\left(\frac{(38.94)^3 l D^2 S^2 A}{\epsilon^3\delta}\right)} $$
\end{proof}


\section{Simplification of $E_1$}
\label{sec:E1}
\begin{align*}
E_1 &= \sqrt{\frac{5}{8} T \log{\left( \frac{8T}{\delta} \right)}} + \sqrt{T} + D\sqrt{\frac{5}{2} T \log{\left( \frac{8T}{\delta} \right)}} \\
	&\leq \sqrt{ \frac{5}{4} T \log{ \left( \frac{T}{\delta} \right)} } + \sqrt{T} +  D\sqrt{5 T \log{\left( \frac{T}{\delta} \right)}} \qquad \text{ if } T \geq 8 \delta  \\
	&\leq D \sqrt{ \frac{5}{4} T \log{ \left( \frac{T}{\delta} \right)} } + D \sqrt{T\log{ \left( \frac{T}{\delta} \right)}} + D\sqrt{5 T \log{\left( \frac{T}{\delta} \right)}} \\
	&\leq 4.36 D \sqrt{T \log{\left(\frac{T}{\delta}\right)}} 
\end{align*}
The first inequality is true because if $T \geq 8\delta$, then $\log{\left( \frac{8T}{\delta} \right)} \leq 2 \log{\left( \frac{T}{\delta} \right)}$.

\section{Simplification of $E_2$}
\label{sec:E2}
\begin{align*}
E_2 &= \left\lceil \frac{T}{W} \right\rceil  DSA \log_2{ \left( \frac{8W}{SA} \right)} + \left( 2D \sqrt{14S \log{\left( \frac{2AT}{\delta} \right) }} + 2\right) (2\sqrt{2} + 2) \left \lceil \frac{T}{W} \right \rceil \sqrt{SAW} \\
	&\leq  \left\lceil \frac{T}{W} \right\rceil DS \sqrt{W \log{\left( \frac{T}{\delta} \right)}} + \left( 2D \sqrt{14S \log{\left( \frac{2AT}{\delta} \right) }} + 2\right) (2\sqrt{2} + 2) \left \lceil \frac{T}{W} \right \rceil \sqrt{SAW}  \\
	&\leq \left\lceil \frac{T}{W} \right\rceil DS \sqrt{W \log{\left( \frac{T}{\delta} \right)}} + \left( 2D \sqrt{28S \log{\left( \frac{T}{\delta} \right) }} + 2\right) (2\sqrt{2} + 2) \left \lceil \frac{T}{W} \right \rceil \sqrt{SAW} \\
&\leq \left(1 + (2\sqrt{28} + 2)\cdot(2\sqrt{2} + 2) \right) \left\lceil \frac{T}{\sqrt{W}} \right\rceil DS \sqrt{ A \log{\left( \frac{T}{\delta} \right)}}  \\
   &\leq 61.76 \left\lceil \frac{T}{\sqrt{W}} \right\rceil DS \sqrt{A \log{ \left( \frac{T}{\delta} \right)}}    
\end{align*}

The first inequality is true assuming $W  \geq \frac{A \left( \log_2{(8W/SA)} \right)^2 }{\log{(T/\delta)}}$ and the second inequality is true assuming $  T \geq 2A\delta$. 
}{} 

\end{document}